\documentclass[12pt]{article}
\usepackage[centertags]{amsmath}
\usepackage{amsfonts}
\usepackage{amssymb}
\usepackage{latexsym}
\usepackage{amsthm}
\usepackage{newlfont}
\usepackage{graphicx}
\usepackage{listings}
\usepackage{booktabs}
\usepackage{abstract}
\newlength{\defbaselineskip}
\newcommand{\setlinespacing}[1]%
           {\setlength{\baselineskip}{#1 \defbaselineskip}}

\newcommand{\bR}{{\mathbb{R}}}

\newcommand{\actaqed}{\hfill $\actabox$}
{\medskip\noindent \textit{Proof of #1. }}%
{\actaqed \medskip}

\def\D{{\mathcal D}}

\def\R{{\mathbb R}}
\def \<{\langle}
\def\>{\rangle}

\def \e{\epsilon}

\def \ff{\varphi}

\def \sp{\operatorname{span}}

\def\bt{\beta}
\def\la{\lambda}
\def\al{\alpha}
\def\ga{\gamma}
\def\ka{\kappa}

\newcommand{\be}{\begin{equation}}
\newcommand{\ee}{\end{equation}}

\newtheorem{Theorem}{Theorem}[section]
\newtheorem{Lemma}{Lemma}[section]

\newtheorem{Proposition}{Proposition}[section]

\numberwithin{equation}{section}

\begin{document}
\title{{Dictionary descent in optimization} }
\author{V.N. Temlyakov \thanks{ University of South Carolina and Steklov Institute of Mathematics. Research was supported by NSF grant DMS-1160841}} \maketitle
\begin{abstract}
{ The problem of convex optimization is studied. Usually in convex optimization the minimization is over a $d$-dimensional domain.  Very often the convergence rate of an optimization algorithm depends on the dimension $d$. The algorithms studied in this paper utilize dictionaries instead of a canonical basis used in the coordinate descent algorithms. We show how this approach allows us to reduce dimensionality of the problem. Also, we investigate which properties of a dictionary are beneficial for the convergence rate of typical greedy-type algorithms. }
\end{abstract}

\section{Introduction. Known results}

{\bf 1.1. The problem setting.}
A typical 
problem of convex optimization is to find an approximate solution to the problem
\begin{equation}\label{1.1}
E^*:=\inf_{x\in D} E(x)
\end{equation}
under assumption that $E$ is a convex function. In the case that we are optimizing over the whole space $X$, it is called an {\it unconstrained optimization problem}. In many cases we are interested either in optimizing over $x$ of special structure or in optimizing over $x$ from a given domain $D$ ({\it constrained optimization problem}).   Usually in convex optimization, the function $E$ is defined on a finite dimensional space $\R^d$ (see \cite{BL}, \cite{N}). Very often the convergence rate of an optimization algorithm depends on the dimension $d$. 
Many contemporary numerical applications ambient space $\R^d$ involves a large dimension $d$ and we would like to obtain bounds on the convergence rate independent of the dimension $d$. Our recent results on optimization on infinite dimensional Banach spaces 
provide such bounds on the convergence rate (see \cite{Tco1}, \cite{Tco2}, \cite{DT1}, \cite{Tco3}). Solving (\ref{1.1}) is an example of a high dimensional problem and is known to suffer the curse of dimensionality without
additional assumptions on $E$ which serve to reduce its dimensionality.  These additional assumptions take the form
of smoothness restrictions on $E$ and assumptions which   imply that the minimum in (\ref{1.1}) is attained on
a   subset of $D$ with additional structure.  Typical assumptions for the latter involve notions of sparsity or compressibility,
which are by now heavily employed concepts for high dimensional problems.  We will always assume that there is   a point $x^*\in D$ where the minimum $E^*$ is attained, $E(x^*)=E^*$. 

The algorithms studied in this paper utilize dictionaries $\D$ of $X$.   We say that a set of elements (functions) $\D$ from $X$ is a dictionary  if each $g\in \D$ has norm bounded by one ($\|g\|\le1$),
and the closure of $\sp \D$ is $X$.  The symmetric dictionary is $\D^\pm :=\{\pm g,g\in \D\}$. We denote the closure (in $X$) of the convex hull of $\D^\pm$ by $A_1(\D)$. In other words $A_1(\D)$ is the closure of conv($\D^\pm$). We use this notation because it has become a standard notation in relevant greedy approximation literature. 
 Given such a dictionary $\D$, there are several types of domains $D$ that are employed in applications.   Sometimes, these domains are the natural domain of the physical problem.   Other times these are constraints imposed on the minimization problem to ameliorate high dimensionality.  We mention the following two common settings.
 
{\bf Sparsity Constraints:}   The set $\Sigma_m(\D)$ of functions
\be
\label{sparse}
f=\sum_{g\in\Lambda}c_gg,\quad |\Lambda|=m,
\ee
is called the set of {\it sparse} functions of order $m$ with respect to the dictionary $\D$.  One common assumption is to minimize $E$ on the domain $D=\Sigma_m(\D)$, i.e. to look for an $m$ sparse minimizer of  (\ref{1.1}).

{\bf $\ell_1$ constraints:}
  A more general setting is to minimize $E$ over the closure $A_1(\D)$ (in $X$)  of the convex hull  of $\D^{\pm}$.    A slightly more general setting is to minimize $E$ over  one of the sets
 \be 
 \label{LM}
 {\mathcal L}_M:=\{g\in X:\ g/M\in A_1(\D)\}.
 \ee

It is pointed out in \cite{FNW} that there has been considerable interest in solving the convex unconstrained optimization problem 
\begin{equation}\label{1.4}
\min_{x}\frac{1}{2}\|y-\Phi x\|_2^2 +\la\|x\|_1
\end{equation}
where $x\in \R^n$, $y\in \R^k$, $\Phi$ is an $k\times n$ matrix, $\la$ is a nonnegative parameter, $\|v\|_2$ denotes the Euclidian norm of $v$, and 
$\|v\|_1$ is the $\ell_1$ norm of $v$. Problems of the form (\ref{1.4}) have become familiar over the past three decades, particularly in statistical and signal processing contexts. Problem (\ref{1.4}) is closely related to the following convex constrained optimization problem
\begin{equation}\label{1.5}
\min_{x}\frac{1}{2}\|y-\Phi x\|_2^2 \quad\text{subject to}\quad \|x\|_1\le A.
\end{equation}
The problem (\ref{1.5}) is a convex optimization problem of the energy function $E(x,\Phi):=\frac{1}{2}\|y-\Phi x\|_2^2$ on the octahedron $\{x:\|x\|_1\le A\}$ in $\R^n$. The domain of optimization is simple and all dependence on the matrix $\Phi$ is in the energy function $E(x,\Phi)$, which makes the problem difficult. Also, the domain is in the high dimensional space $\R^n$. In typical applications, for instance in compressed sensing,  $k$ is much smaller than $n$. 
We recast the above problem as an optimization problem with $E(z):=\frac{1}{2}\|y-z\|_2^2$ over the domain $A_1(\D)$
 with respect to a dictionary $\D:=\{\pm\varphi_i\}_{i=1}^n$, which is associated with a $k\times n$ matrix
$\Phi=[\varphi_1\dots\varphi_n]$ with $\varphi_j\in \R^k$ being the column vectors of $\Phi$.
In this formulation the energy function $E(z)$ is very simple and all dependence on $\Phi$ is in the form of the domain $A_1(\D)$. Other important feature of the new formulation is that optimization takes place in the $\R^k$ with relatively small $k$. The simple form $E(z)=\frac{1}{2}\|y-z\|_2^2$ of the energy function shows that the minimization of $E(z)$ over $\Sigma_m(\D)$ is equivalent to $m$-sparse approximation of $y$ with respect to $\D$.
 The condition $y\in A_1(\D)$ is
equivalent to existence of $x\in \R^n$ such that $y=\Phi x$ and
\begin{equation}\label{1.6}
\|x\|_{1}:=|x_1|+\dots+|x_n| \le 1.
\end{equation}
As a direct corollary of Theorem \ref{T1.1} (see below), we get for any
$y\in A_1(\D)$ that the WCGA and the WGAFR with $\tau=\{t\}$ guarantee the
following upper bound for the error
\begin{equation}\label{1.7}
\|y_m\|_2\le Cm^{-1/2},
\end{equation}
where $y_m$ is the residual after $m$ iterations.
The bound (\ref{1.7}) holds for any $\D$ (any $\Phi$).

We concentrate on the study of the following optimization problem
\be\label{1.8}
\inf_{x\in D} E(x), \qquad D:=\{x:E(x)\le E(0)\}\subset \R^d.
\ee
It is clear that
$$
\inf_{x\in D} E(x) = \inf_{x\in \R^d} E(x).
$$
Therefore, the optimization problem (\ref{1.8}), which is formulated as a constrained optimization problem, is equivalent to the unconstrained optimization problem.
We apply greedy algorithms with respect to dictionaries with special properties and try to understand how these special properties effect dependence on $d$ of the convergence rate. \newline
{\bf 1.2. Greedy approximation in Banach spaces.} We begin with a brief description of greedy approximation methods in Banach spaces. The reader can find a detailed discussion of greedy approximation in the book \cite{Tbook}. 

A typical problem of sparse approximation is the following. Let $X$ be a Banach space with norm $\|\cdot\|$ and $\D$ be a set of elements of $X$. For a given $\D$ consider the set $\Sigma_m(\D)$ of all $m$-term linear combinations with respect to $\D$ ($m$-sparse with respect to $\D$)   defined above.
We are interested in approximation of a given $f\in X$ by elements of $\Sigma_m(\D)$.
The best we can do is
\begin{equation}\label{0.0}
\sigma_m(f,\D) := \inf_{x\in\Sigma_m(\D)}\|f-x\|.
\end{equation}
Greedy algorithms in approximation theory are designed to provide a simple way to build good approximants of $f$ from $\Sigma_m(\D)$. Clearly, problem (\ref{0.0}) is an optimization problem of $E_f(x):=\|f-x\|$ over the manifold $\Sigma_m(\D)$.

 For a nonzero element $f\in X$ we let $F_f$ denote a norming (peak) functional for $f$ that is a functional with the following properties 
$$
\|F_f\| =1,\qquad F_f(f) =\|f\|.
$$
The existence of such a functional is guaranteed by the Hahn-Banach theorem. The norming functional $F_f$ is a linear functional (in other words is an element of the dual to $X$ space $X^*$) which can be explicitly written in some cases. In a Hilbert space $F_f$ can be identified with $f\|f\|^{-1}$. In the real $L_p$, $1<p<\infty$, it can be identified with $f|f|^{p-2}\|f\|_p^{1-p}$. 
We describe a typical greedy algorithm which uses a norming functional. We call this  family of algorithms {\it dual greedy algorithms}. 
Let 
$t\in(0,1]$ be a given weakness parameter. We first  define the Weak Chebyshev Greedy Algorithm (WCGA) (see \cite{T15}) that is a generalization for Banach spaces of the Weak Orthogonal Greedy Algorithm (WOGA), which is known under the name Weak Orthogonal Matching Pursuit in signal processing.   

 {\bf Weak Chebyshev Greedy Algorithm (WCGA).} Let  $t\in(0,1]$ be a weakness  parameter.
We take $f_0\in X$. Then for each $m\ge 1$ we have the following inductive definition.

(1) $\varphi_m :=\varphi^{c,t}_m \in \D$ is any element satisfying
$$
|F_{f_{m-1}}(\varphi_m)| \ge t  \sup_{g\in\D}|F_{f_{m-1}}(g)|.
$$

(2) Define
$$
\Phi_m := \Phi^t_m := \sp \{\varphi_j\}_{j=1}^m,
$$
and define $G_m := G_m^{c,t}$ to be the best approximant to $f_0$ from $\Phi_m$.

(3) Let
$$
f_m := f^{c,t}_m := f_0-G_m.
$$
The index $c$ in the notation refers to Chebyshev. We use the name Chebyshev in this algorithm because at step (2) of the algorithm we use best approximation operator which bears the name of the {\it Chebyshev projection} or the {\it Chebyshev operator}. In the case of Hilbert space the Chebyshev projection is the orthogonal projection and it is reflected in the name of the algorithm. We use notation $f_m$ for the residual of the algorithm after $m$ iterations. This standard in approximation theory notation is justified by the fact that we interpret $f_0$ as a residual after $0$ iterations and iterate the algorithm replacing $f_0$ by $f_1$, $f_2$, and so on. In signal processing the residual after $m$ iterations is often denoted by $r_m$ or $r^m$.   

  For a Banach space $X$ we define the modulus of smoothness
$$
\rho(u) := \sup_{\|x\|=\|y\|=1}(\frac{1}{2}(\|x+uy\|+\|x-uy\|)-1).
$$
The uniformly smooth Banach space is the one with the property
$$
\lim_{u\to 0}\rho(u)/u =0.
$$
 
The following proposition is well-known (see, \cite{Tbook}, p.336).
\begin{Proposition}\label{P1.1} Let $X$ be a uniformly smooth Banach space. Then, for any $x\neq0$ and $y$ we have
\begin{equation}\label{1.9}
F_x(y)=\left(\frac{d}{du}\|x+uy\|\right)(0)=\lim_{u\to0}(\|x+uy\|-\|x\|)/u. 
\end{equation}
\end{Proposition}
Proposition \ref{P1.1} shows that in the WCGA we are looking for an element $\ff_m\in\D$ that provides a big derivative of the quantity $\|f_{m-1}+u\ff_m\|$. 
Here is one more important greedy algorithm. 

  {\bf Weak Greedy Algorithm with Free Relaxation  (WGAFR).} 
Let  $t\in(0,1]$ be a weakness  parameter. We define for $f_0  \in X$, $G_0  := 0$. Then for each $m\ge 1$ we have the following inductive definition.

(1) $\varphi_m   \in \D$ is any element satisfying
$$
|F_{f_{m-1}}(\varphi_m  )| \ge t   \sup_{g\in\D}|F_{f_{m-1}}(g)|.
$$

(2) Find $w_m$ and $ \lambda_m$ such that
$$
\|f-((1-w_m)G_{m-1} + \la_m\varphi_m)\| = \inf_{ \la,w}\|f-((1-w)G_{m-1} + \la\varphi_m)\|
$$
and define
$$
G_m:=   (1-w_m)G_{m-1} + \la_m\varphi_m.
$$

(3) Let
$$
f_m   := f-G_m.
$$
It is known that both algorithms WCGA and WGAFR converge in any uniformly smooth Banach space. The following theorem provides rate of convergence (see \cite{Tbook}, pp. 347, 353). 
\begin{Theorem}\label{T1.1} Let $X$ be a uniformly smooth Banach space with modulus of smoothness $\rho(u)\le \gamma u^q$, $1<q\le 2$. Take a number $\e\ge 0$ and two elements $f_0$, $f^\e$ from $X$ such that
$$
\|f_0-f^\e\| \le \e,\quad
f^\e/B \in A_1(\D),
$$
with some number $B=C(f,\e,\D,X)>0$.
Then, for both algorithms WCGA and WGAFR  we have  
$$
\|f_m\| \le  \max\left(2\e, C(t,q,\gamma)(B+\e) m^{-(q-1)/q}\right) . 
$$
\end{Theorem}

The above Theorem \ref{T1.1} simultaneously takes care of two issues: noisy data and approximation in an interpolation space. In order to apply it for noisy data we interpret $f_0$ as a noisy version of a signal and $f^\e$ as a noiseless version of a signal. Then, assumption $f^\e/B\in A_1(\D)$ describes our smoothness assumption on the noiseless signal. Theorem \ref{T1.1} can be applied for approximation of $f_0$ under assumption that $f_0$ belongs  to one  
of interpolation spaces between $X$ and the space generated by the $A_1(\D)$-norm (atomic norm).   \newline
{\bf 1.3. Greedy algorithms for convex optimization.} In \cite{Tco1} we generalized the algorithms WCGA and WGAFR to the case of convex optimization and proved an analog of Theorem \ref{T1.1} for the new algorithms. Let us illustrate this on the generalization of the WGAFR.

We 
assume that the set
$$
D:=\{x:E(x)\le E(0)\}
$$
is bounded.
For a bounded set $D$ define the modulus of smoothness of $E$ on $D$ as follows
\begin{equation}\label{c1.1}
\rho(E,u):=\frac{1}{2}\sup_{x\in D, \|y\|=1}|E(x+uy)+E(x-uy)-2E(x)|.
\end{equation}
A typical assumption in convex optimization is of the form ($\|y\|=1$)
$$
|E(x+uy)-E(x)-\<E'(x),uy\>| \le Cu^2
$$
which corresponds to the case $\rho(E,u)$ of order $u^2$.
We assume that $E$ is Fr{\'e}chet differentiable. Then convexity of $E$ implies that for any $x,y$ 
\begin{equation}\label{c1.2}
E(y)\ge E(x)+\<E'(x),y-x\>
\end{equation}
or, in other words,
\begin{equation}\label{c1.3}
E(x)-E(y) \le \<E'(x),x-y\> = \<-E'(x),y-x\>.
\end{equation} 
The above assumptions that $E$ is a Fr{\'e}chet differentiable convex function and $D$ is a bounded domain guarantee that $\inf_xE(x) >-\infty$. 
We will use the following simple lemma.
\begin{Lemma}\label{L1.1} Let $E$ be Fr{\'e}chet differentiable convex function. Then the following inequality holds for $x\in D$
\begin{equation}\label{c1.6}
0\le E(x+uy)-E(x)-u\<E'(x),y\>\le 2\rho(E,u\|y\|).  
\end{equation}
\end{Lemma}

  {\bf Weak Greedy Algorithm with Free Relaxation  (WGAFR(co)).} 
Let  $t\in(0,1]$ be a weakness  parameter. We define   $G_0  := 0$. Then for each $m\ge 1$ we have the following inductive definition.

(1) $\varphi_m   \in \D$ is any element satisfying
$$
|\<-E'(G_{m-1}),\varphi_m\>| \ge t  \sup_{g\in\D}|\<-E'(G_{m-1}),g\>|.
$$

(2) Find $w_m$ and $ \lambda_m$ such that
$$
E((1-w_m)G_{m-1} + \la_m\varphi_m) = \inf_{ \la,w}E((1-w)G_{m-1} + \la\varphi_m)
$$
and define
$$
G_m:=   (1-w_m)G_{m-1} + \la_m\varphi_m.
$$
We keep in algorithms for convex optimization notation $G_m$ ($G$ comes from Greedy) used in greedy approximation algorithms   to stress that an $m$th approximant $G_m$ is obtained by a greedy algorithm. Standard optimization theory notation for it is $x^{m}$. 
 The following theorem is from \cite{Tco1}.

\begin{Theorem}\label{T1.2} Let $E$ be a uniformly smooth convex function with modulus of smoothness $\rho(E,u)\le \gamma u^q$, $1<q\le 2$. Take a number $\e\ge 0$ and an element  $f^\e$ from $D$ such that
$$
E(f^\e) \le \inf_{x\in D}E(x)+ \e,\quad
f^\e/B \in A_1(\D),
$$
with some number $B=C(E,\e,\D)\ge 1$.
Then we have for the WGAFR(co)  
$$
E(G_m)-\inf_{x\in D}E(x) \le  \max\left(2\e, C_1(t,q,\gamma)B^q m^{1-q}\right) . 
$$
\end{Theorem}

 We note that in all algorithms studied in this paper the sequence $\{G_m\}_{m=0}^\infty$ of approximants satisfies the conditions
 $$
 G_0=0,\quad E(G_0)\ge E(G_1) \ge E(G_2) \ge \dots .
 $$
 This guarantees that $G_m\in D$ for all $m$.

\section{Sparse approximation with respect to a dictionary}

{\bf 2.1. Exponential decay of errors.} We begin with showing that in the finite dimensional Banach space $X$ a number of greedy-type algorithms with respect to a finite dictionary 
provide exponential (linear) decay of error. From the definition of modulus of smoothness of $X$ we obtain for any $\la$
\begin{equation}\label{2.1}
 \|f_{m-1}-\la \ff_m\|+\|f_{m-1}+\la \ff_m\| \le 2\|f_{m-1}\|(1+\rho(\la/\|f_{m-1}\|)).
 \end{equation}
 Next, for $\la\ge 0$
 \begin{equation}\label{2.2}
 \|f_{m-1}+\la\ff_m\| \ge F_{f_{m-1}}(f_{m-1}+\la\ff_m)\ge \|f_{m-1}\| + \la F_{f_{m-1}}(\ff_m).
 \end{equation}
 By the definition of $\ff_m$ in dual-type weak greedy algorithms we have
  \begin{equation}\label{2.3}
 |F_{f_{m-1}}(\ff_m)| \ge t\sup_{g\in\D}|F_{f_{m-1}}(g)|.
 \end{equation}
 Then either
  \begin{equation}\label{2.3'}
 F_{f_{m-1}}(\ff_m) \ge t\sup_{g\in\D}|F_{f_{m-1}}(g)|
 \end{equation}
 or
 $$
 F_{f_{m-1}}(-\ff_m) \ge t\sup_{g\in\D}|F_{f_{m-1}}(g)|.
 $$
 We treat the case, when (\ref{2.3'}) holds. The other case is treated in the same way.
 
 We consider a finite dictionary $\D$ of $X$ being $\R^d$ equipped with the norm $\|\cdot\|:=\|\cdot\|_X$. Introduce the quantity
  \begin{equation}\label{2.4}
 \beta(\D,X):= \min_{F:\|F\|_{X^*}=1}\max_{g\in \D} |F(g)|.
 \end{equation}
 We note that $\beta(\D^\pm,X)=\beta(\D,X)$, where as above $\D^\pm := \{\pm g, g\in \D\}$.
 It is well known and easy to see that $\beta(\D,X)>0$. Indeed, $\max_{g\in\D}|F(g)|$ is a continuous function on $F$ and, therefore, it attains its minimum value on the compact set $\{F: \|F\|_{X^*}=1\}$. Let $\bt(\D,X)=\max_{g\in\D}|F_0(g)|$, $\|F_0\|_{X^*}=1$. Then, assuming $\bt(\D,X)=0$, we
 obtain that $F_0=0$, which is a contradiction. 
 
 Thus, we have
 $$
 F_{f_{m-1}}(\ff_m) \ge t\bt(\D,X) >0.
 $$
 For $\rho(u)\le \gamma u^q$, $q\in (1,2]$, we obtain from (\ref{2.1}) and (\ref{2.2})
  \begin{equation}\label{2.5}
 \|f-\la\ff_m\| \le \|f_{m-1}\| +2\ga \|f_{m-1}\|(\la/\|f_{m-1}\|)^q -\la t \bt, \quad \bt:=\bt(\D,X).
 \end{equation}
 Setting 
 $$
 \kappa:= \frac{1}{2}\left(\frac{t\bt}{4\ga}\right)^{\frac{1}{q-1}},\quad \la_1:= 2\ka \|f_{m-1}\|
 $$
 we get
  \begin{equation}\label{2.6}
 \|f_{m-1}-\la_1\ff_m\| \le \|f_{m-1}\|(1-\ka t\bt).
 \end{equation}
 Bound (\ref{2.6}) guarantees exponential (linear) decay for the following greedy-type algorithms   WCGA, WGAFR, defined above, and for the Weak Dual Greedy Algorithm (WDGA), $X$-Greedy Algorithm (XGA), defined in \cite{Tbook}, Ch.6 
 $$
 \|f_m\|\le (1-\ka t\bt)^m\|f_0\|.
 $$
 A very important issue is how $\bt(\D,X)$ depends on dimension $d$. Let us first consider the case of $\D=\{\pm e^j\}_{j=1}^d$ where $\{e^j\}_{j=1}^d$ is a canonical orthonormal basis for $X$ being the Euclidean space $\ell_2^d$. Then it is easy to see that $\bt(\D,\ell_2^d)=d^{-1/2}$ and the rate of convergence provided by (\ref{2.6}) is of order $(1- C(t)/d)^m$. 
 
 The error rate $(1-C/d)^m$ does not provide good approximation for small $m\le d$. However, there exists a different technique that provides reasonable error decay for small $m$ under extra assumptions on $f_0$.  Theorem \ref{T1.1} from Introduction provides such rate of convergence.   \newline
{\bf 2.2. Incoherent dictionaries.} In the discussion above we considered approximation with respect to an arbitrary dictionary $\D$. If we can choose a dictionary $\D$ satisfying certain conditions then we can improve on the error decay rate for small $m$. We give an example from \cite{T144}. 

{\bf A.} Let $\D:=\{g_i\}_{i=1}^\infty$. We say that $f=\sum_{i\in T}x_ig_i$, $g_i\in \D$, $i\in T$, has $\ell_1$ incoherence property with parameters $D$, $V$, and $r$ if for any $A\subset T$ and any $\Lambda$ such that $A\cap \Lambda =\emptyset$, $|A|+|\Lambda| \le D$ we have for any $\{c_i\}$
\begin{equation}\label{2.7}
\sum_{i\in A}|x_i| \le V|A|^r\|f_A-\sum_{i\in\Lambda}c_ig_i\|,\quad f_A:=\sum_{i\in A} x_ig_i.
\end{equation}
A dictionary $\D$ has $\ell_1$ incoherence property with parameters $K$, $D$, $V$, and $r$ if for any $A\subset B$, $|A|\le K$, $|B|\le D$ we have for any $\{c_i\}_{i\in B}$
$$
\sum_{i\in A} |c_i| \le V|A|^r\|\sum_{i\in B} c_ig_i\|.
$$

  \begin{Theorem}\label{T2.2} Let $X$ be a Banach space with $\rho(u)\le \gamma u^q$, $1<q\le 2$. Suppose $K$-sparse $f^\e$ satisfies {\bf A}  and $\|f_0-f^\e\|\le \e$. Then the WCGA with weakness parameter $t$ applied to $f_0$ provides $(q':=q/(q-1)$)
$$
\|f_{C(t,\gamma,q)V^{q'}\ln (VK) K^{rq'}}\| \le C\e\quad\text{for}\quad K+C(t,\gamma,q)V^{q'}\ln (VK) K^{rq'}\le D
$$
with an absolute constant $C$ and $C(t,\gamma,q) = C_2(q)\gamma^{\frac{1}{q-1}}  t^{-q'}$.
\end{Theorem}

In the above setting we have the following trade off of different features of the dictionary $\D$. Clearly, parameter $\bt(\D,X)$ depends on $\D$. We may expect that for good dictionaries $\bt(\D,X)$ grows with the growth of the cardinality of $\D$. Also, the set $A_1(\D)$, for which we can guarantee good rate of convergence for small $m$, grows when cardinality of $\D$ increases. These are arguments for choosing a large dictionary $\D$ for approximation. However, really large dictionaries, say, with cardinality exponential in $d$, are not manageable computationally. Therefore, a reasonable compromise seems to be to find good dictionaries with cardinality polynomial in $d$. We discuss this in detail on the example of the Euclidean space $\ell_2^d$. It is known that dictionaries with small coherence are good for sparse approximation. Also, there are different explicit deterministic constructions of large dictionaries with small coherence. We now present some known results on incoherent dictionaries (for more detail see \cite{Tbook}, Section 5.7). 

  Let $\D=\{g^k\}_{k=1}^N$ be a normalized ($\|g^k\|=1$, $k=1,\dots,N$) system of vectors in $\R^d$ or ${\mathbb C}^n$ equipped with the Euclidean norm.
We define the coherence parameter of the dictionary $\D$ as follows
$$
M(\D) := \sup_{k\neq l} |\<g^k,g^l\>|.
$$
In this section we discuss the following   characteristics
$$
N(d,\mu) := \sup\{N:\exists \D \quad\text{such that} \quad |\D| \ge N, M(\D)\le\mu\}.
$$

The problem of studying $N(d,\mu)$ is equivalent to a fundamental problem of information theory. It is a problem on optimal spherical codes. A spherical code ${\mathcal S}(d,N,\mu)$ is a set of $N$ points (code words) on the $d$-dimensional unit sphere, such that the absolute values of inner products between any two code words is not greater than $\mu$. The problem is to find the largest $N^*$ such that the spherical code  ${\mathcal S}(d,N^*,\mu)$ exists. It is clear that $N^*=N(d,\mu)$. Here is the known results on the behavior of $N(d,\mu)$ (see \cite{Tbook}, p.316).

\begin{Theorem} \label{T2.3}There exist two positive constants $C_1$ and $C_2$ such that for $\mu\in [(2d)^{-1/2},1/2]$ we have 
$$
C_2\exp(d\mu^2/2)\le N(d,\mu)\le \exp(C_1d\mu^2\ln(2/\mu)).
$$
\end{Theorem}
A very interesting and difficult problem is to provide an explicit (deterministic) construction of a large system with small coherence. We addressed this problem in \cite{NT}. In \cite{NT} (see also \cite{Tbook}, Section 5.7) we used Weil's sums to construct a system $\D$ in ${\mathbb C}^d$ with $M(\D)\le \mu$ of 
cardinality of order $\exp(\mu d^{1/2}\ln d)$. We note that Kashin \cite{Ka3} constructed a system $\D$ in $\bR^d$ with $M(\D)\le\mu$ of cardinality of order  $\exp(C\mu d^{1/2}\ln d)$ using symbols of Legendre. Similar results have been obtained in \cite{GMS} with combinatorial designs and in \cite{D3} with the finite fields technique.

Theorem \ref{T2.3} shows that if we are interested in maximal cardinality of dictionaries with the coherence parameter $\mu$ then the polynomial in $d$ constraint on the dictionary cardinality implies $\mu \le C\left(\frac{\ln d}{d}\right)^{1/2}$. The discussion after Theorem \ref{T2.3} shows that we can build dictionaries $\D$ with polynomial in $d$ cardinality with $M(\D)\le Cd^{-1/2}$. Assume that $\D$ is such a dictionary. Then known results on the Lebesgue-type inequalities (see \cite{Tbook}, Section 2.6) for the OGA and the WOGA (see \cite{Zhang}) provide
\begin{equation}\label{2.8} 
\|f_{C_1m}\|_2 \le C_2\sigma_m(f_0,\D)_2,\qquad m\le C_3 M(\D)^{-1},\quad M(\D)^{-1}\ge C^{-1} d^{1/2}.
\end{equation}
Thus for the number of iterations $m$ up to $O(d^{1/2})$ the WOGA works almost ideally and provides near best $m$-term approximation. If, in addition, $f_0\in A_1(\D)$ then for all $m$
$$
\|f_m\|_2 \le C(t)m^{-1/2},
$$
which we can use for $m>C_4d^{1/2}$. Further, if we assume that $\D=\D(\mu)$ is the extremal dictionary for a given $\mu$: $M(\D)\le\mu$ and $|\D(\mu)|=N(d,\mu)$ then, as in subsection 2.1, we can guarantee the exponential decay of the error. Indeed, our assumption that $\D(\mu)$ is an extremal dictionary for $\mu$ implies that for any $x\in \ell_2^d$ there is $g^k\in\D(\mu)$ such that $|\<x/\|x\|_2,g^k\>|>\mu$. This means that
$$
\bt(\D(\mu),\ell_2^d)\ge \mu=Cd^{-1/2}.
$$
Therefore, for any of the algorithms mentioned above we obtain
$$
\|f_m\|_2\le (1-C/d)^{m-n}\|f_n\|_2,\qquad 1\le n\le m.
$$
We note that for the WOGA the residual $f_d$ is equal to zero. For large $m$ the WOGA is more difficult computationally then the Weak Greedy Algorithm (WGA) (see \cite{Tbook}, Ch. 2). Therefore, we propose the following strategy. For the dictionary $\D(\mu)$ apply the WOGA for $n\le C_3C^{-1}d^{1/2}$ iterations. It will give a near optimal error:
$$
\|f_n\|_2\le C_2\sigma_{n/C_3}(f_0,\D(\mu)_2.
$$
Then switch to the WGA and get the error
$$
\|f_m\|_2 \le C'\sigma_{n/C_3}(f_0,\D(\mu)_2(1-C/d)^m.
$$
{\bf 2.3. Dictionaries associated with covering.} We now discuss dictionaries, which are related to covering a unit ball $B_2^d$ by balls of radius $r<1$. The reader can find some related results for finite dimensional Banach spaces in \cite{T6}.
Let  $X$ be a Banach space $\R^d$ with a norm $\|\cdot\|$ and let $B:=B_X$ denote the corresponding closed unit ball:
\begin{equation}\label{2.9}
B:=B_X:=\{x\in \R^d:\|x\|\le 1\}.
\end{equation}
Notation $B(x,r):=B_X(x,r)$  will be used respectively for closed  balls with the center $x$ and radius $r$. In case $r=1$ we drop it from the notation: $B(x):=B(x,1)$. In the case $X=\ell_2^d$ we write $B_2^d:=B_{\ell_2^d}$. Also, define $S^{d-1}:=\{x\in \ell_2^d: \|x\|_2=1\}$ the unit sphere. For a compact set $A$ and a positive number $\e$ we define the covering number $N_\e(A)$
 as follows
$$
N_\e(A) := N_\e(A,X) 
:=\min \{n \,|\, \exists x^1,\dots,x^n : A\subseteq \cup_{j=1}^n B_X(x^j,\e)\}.
$$
The following proposition is well known.
\begin{Proposition}\label{P2.1} For any $d$-dimensional Banach space $X$ we have
$$
\e^{-d} \le N_\e(B_X,X) \le (1+2/\e)^d.
$$
\end{Proposition}
We begin with two simple lemmas. 
\begin{Lemma}\label{L2.1} Let $\cup_{j=1}^N B^d_2(x^j,r)$ be a covering of $S^{d-1}$:
$$
S^{d-1}\subset \cup_{j=1}^N B^d_2(x^j,r).
$$
Then for $\D:=\{g^j\}_{j=1}^N$, $g^j:=x^j/\|x^j\|_2$, we have
$$
\bt:=\bt(\D,\ell_2^d) \ge (1-r^2)^{1/2}.
$$
\end{Lemma}
\begin{proof} Take an arbitrary $x\in S^{d-1}$. Then there exists $x^k$, $1\le k\le N$, such that 
$\|x-x^k\|_2 \le r$. Further, 
$$
\|x-\<x,g^k\>g^k\|_2 \le \|x-x^k\|_2 \le r
$$
and
$$
r^2 \ge \|x-\<x,g^k\>g^k\|_2^2 = \|x\|^2 - \<x,g^k\>^2.
$$
This implies
$$
\<x,g^k\>^2 \ge 1-r^2,
$$
which completes the proof of Lemma \ref{L2.1}.
\end{proof}

\begin{Lemma}\label{L2.2} Let $\D=\{g^j\}_{j=1}^N$ have $\bt:=\bt(\D,\ell_2^d)\le 2^{-1/2}$. Then 
$$
\left(\cup_{j=1}^N B^d_2(x^j,r)\right)\cup\left(\cup_{j=1}^N B^d_2(-x^j,r)\right) 
$$ 
with $x^j:=\bt g^j$ and $r^2=1-\bt^2$ is a covering of $B^d_2$. 
\end{Lemma}
\begin{proof} Take any $x\in S^{d-1}$. It follows from the definition of $\bt$ that there exists $g^k\in \D$ such that $|\<x,g^k\>| \ge \bt.$ Then either $\<x,g^k\> \ge \bt$ or $\<x,-g^k\> \ge \bt$. We treat the case $\<x,g^k\> \ge \bt$. The other case is treated in the same way. 
For $\al\in[0,1]$ estimate
$$
\|\al x-\bt g^k\|_2^2 = \al^2 -2\al\bt\<x,g^k\> +\bt^2 \le \al^2 -2\al\bt^2 + \bt^2
$$
$$
\le \al -2\al\bt^2 + \bt^2 = \al(1-2\bt^2)+2\bt^2 -\bt^2 \le 1-\bt^2. 
$$
Thus the segment $[0,x]$ is covered by $B_2^d(\bt g^k,r)$. This completes the proof of Lemma \ref{L2.2}. 
\end{proof}
We now use the above lemmas and known results on covering to find the right behavior of $\bt(\D,\ell_2^d)$ for dictionaries of polynomial in $d$ cardinality. Define for $a\ge 1$
$$
\bt(a):= \sup_{\D:|\D|\le d^a}\bt(\D,\ell_2^d).
$$
\begin{Theorem}\label{T2.4} We have
$$
\bt(a) \le \left(\frac{2(a\ln d+\ln 2)}{d}\right)^{1/2}.
$$
\end{Theorem}
\begin{proof} By Lemma \ref{L2.2} a dictionary with $\bt:=\bt(\D,\ell_2^d)\le 2^{-1/2}$ provides a covering of $B^d_2$ by balls of radius $r = (1-\bt^2)^{1/2}$. Combining Proposition \ref{P2.1} (lower bound) and our assumption that $|\D| \le d^a$ we obtain
$$
(1-\bt^2)^{-d/2} \le 2d^a,
$$
$$
\ln(1-\bt^2)\ge -\frac{2(a\ln d+\ln 2)}{d}.
$$
Using the inequality $\ln (1-u) \le -u$, $u\in[0,1)$, we get
$$
\bt^2 \le \frac{2(a\ln d+\ln 2)}{d},
$$
which proves Theorem \ref{T2.4}.
\end{proof}
\begin{Theorem}\label{T2.5} Let $b>0$ be such that $\frac{b\ln d}{d} \le \frac{1}{2}$. There exists $\D$ such that $|\D|\le Cd^{b+5/2}$ and 
$$
\bt(\D,\ell_2^d) \ge \left(\frac{b\ln d}{d}\right)^{1/2}.
$$
\end{Theorem}
\begin{proof} The classical result of Rogers (see \cite{Ro}, Theorem 3) implies the following bound
\begin{equation}\label{2.10}
N_r(B_2) \le Cd^{5/2}r^{-d} .
\end{equation}
The reader can find further slight improvements of (\ref{2.10}) in \cite{VG} and \cite{Du}. 
The bound (\ref{2.10}) is sufficient for our purposes. Set $\bt^2=\frac{b\ln d}{d}$, $r^2=1-\bt^2$.
Then by Lemma \ref{L2.1} the dictionary $\D$ associated with the above covering satisfies 
$\bt(\D,\ell_2^d)\ge \bt$. Also, by (\ref{2.10})
$$
|\D| \le Cd^{5/2}(1-\bt^2)^{-d/2}.
$$
Using the inequality $\ln(1-u)\ge -2u$, $u\in [0,1/2]$ we obtain
$$
(1-\bt^2)^{-d/2}\le e^{\bt^2d} = d^b.
$$
Thus
$$
|\D| \le Cd^{b+5/2}.
$$
\end{proof}

\section{Sparse optimization with respect to a dictionary}

{\bf 3.1. Special energy function.} We begin with the definition of the Weak Chebyshev Greedy Algorithm (WCGA(co)) for convex optimization (see \cite{Tco1}).

 {\bf Weak Chebyshev Greedy Algorithm (WCGA(co)).} Let   $t\in (0,1]$  be a weakness parameter. 
We define $G_0 := 0$. Then for each $m\ge 1$ we have the following inductive definition.

(1) $\varphi_m :=\varphi^{c,t}_m \in \D$ is any element satisfying
$$
|\<-E'(G_{m-1}),\varphi_m\>| \ge t  \sup_{g\in \D}|\< -E'(G_{m-1}),g\>|.
$$

(2) Define
$
\Phi_m := \Phi^t_m := \sp \{\varphi_j\}_{j=1}^m,
$
and define $G_m := G_m^{c,t}$ to be the point from $\Phi_m$ at which $E$ attains the minimum:
$
E(G_m)=\inf_{x\in \Phi_m}E(x).
$

First, consider a special case $E(x):=V(\|x-f_0\|)$ with $V$ differentiable on $(0,\infty)$ and $V'(u)>0$ on $(0,\infty)$. Then
\begin{equation}\label{3.1}
-E'(x) = V'(\|x-f_0\|)F_{f_0-x}.
\end{equation} 
Therefore, the greedy step: $\ff_m\in\D$ is any element satisfying 
\begin{equation}\label{3.2}
|\<-E'(G_{m-1}),\ff_m\>| \ge t \sup_{g\in \D}|\<-E'(G_{m-1},g\>|
\end{equation}
is equivalent to
\begin{equation}\label{3.3}
|\<F_{f_0-G_{m-1}},\ff_m\>| \ge t\sup_{g\in\D} |\<F_{f_0-G_{m-1}},g\>|.
\end{equation}
Relation (\ref{3.3}) is exactly the greedy step of a dual greedy-type algorithm. 

The Chebyshev step of the WCGA(co): find $G_m\in \Phi_m$ such that
\begin{equation}\label{3.4}
E(G_m)=\min_{x\in\Phi_m}E(x)
\end{equation}
is equivalent to: find $G_m\in\Phi_m$ such that 
$$
\|f_0-G_m\| = \min_{x\in\Phi_m}\|f_0-x\|.
$$
Thus, the WCGA(co) with $E(x)=V(\|x-f_0\|)$ is equivalent to the WCGA applied to $f_0$.
In the same way we obtain that the WGAFR(co) applied to $E(x)=V(\|x-f_0\|)$ is equivalent to the realization of the WGAFR for $f_0$. Therefore, results on approximation of $f_0$, in particular results of Section 2, apply in this case. \newline
{\bf 3.2. Convex energy function.} Second, we discuss a setting where we assume that $E(x)$ is a Fr{\'e}chet differentiable convex function with modulus of smoothness $\rho(u)\le \ga u^q$, $q\in(1,2]$ on the domain
$$
D:=\{x:E(x)\le E(0)\}.
$$
We assume that   $D$ is a bounded domain: 
for all $x\in D$ we have $\|x\|\le C_0$. We begin our discussion with a general case of Banach space $X$. 

\begin{Theorem}\label{T3.1} Let $E$ be a uniformly smooth convex function with modulus of smoothness $\rho(E,u)\le \gamma u^q$, $1<q\le 2$ on the domain $D$. Let $\bt:=\bt(\D,X)>0$ be the parameter defined in (\ref{2.4}). 
Then we have for the WCGA(co) and the WGAFR(co) 
\begin{equation}\label{3.5}
E(G_m)-\inf_{x\in D}E(x) \le  C(t,\ga,C_0)\bt^{-q}m^{1-q} . 
\end{equation}
\end{Theorem}
\begin{proof} Lemma \ref{L1.1} implies for any positive $\la$
\begin{equation}\label{3.7}
E(G_{m-1}+\la\ff_m)\le E(G_{m-1}) -\la \<-E'(G_{m-1}),\ff_m\> +2\rho(E,\la).
\end{equation}
For the dual-type greedy algorithm we have
\begin{equation}\label{3.8}
|\<-E'(G_{m-1}),\ff_m\>| \ge t\sup_{g\in\D}|\<-E'(G_{m-1}),g\>|.
\end{equation}
Using the definition of $\bt(\D,X)$ from (\ref{2.4}) we get either
\begin{equation}\label{3.9}
\<-E'(G_{m-1}),\ff_m\> \ge t\bt \|E'(G_{m-1})\|_{X^*}
\end{equation}
or
$$
\<-E'(G_{m-1}),-\ff_m\> \ge t\bt \|E'(G_{m-1})\|_{X^*}.
$$
We treat the case (\ref{3.9}). The other case is treated in the same way. 
Define
$$
w:=\inf_{x\in D}E(x),\qquad a_m:=E(G_m)-w.
$$
Then convexity of $E$ implies that for any $x,y$ 
\begin{equation}\label{3.10}
E(y)\ge E(x)+\<E'(x),y-x\>.
\end{equation}
Therefore,
\begin{equation}\label{3.11}
a_{m-1}\le \|E'(G_{m-1})\|_{X^*}\sup_{f\in D}\|G_{m-1}-f\|_X \le \|E'(G_{m-1})\|_{X^*} 2C_0.
\end{equation}
Thus,
$$
\|E'(G_{m-1})\|_{X^*} \ge a_{m-1}(2C_0)^{-1}.
$$
Inequalities (\ref{3.7})--(\ref{3.9}) give
$$
E(G_{m-1}+\la\ff_m) -w \le a_{m-1} -\la t\bt a_{m-1}(2C_0)^{-1}+2\ga\la^q.
$$
Choose $\la_1$ from the equation
$$
\la t\bt a_{m-1}(2C_0)^{-1}=4\ga\la^q, \quad \la_1 = \left(\frac{t\bt a_{m-1}}{8\ga C_0}\right)^{\frac{1}{q-1}}.
$$
Then
\begin{equation}\label{3.12}
a_m\le E(G_{m-1}+\la_1\ff_m)-w\le a_{m-1}(1-a_{m-1}^{\frac{1}{q-1}}/B),
\end{equation}
where
\begin{equation}\label{3.13}
B^{-1} = \left(\frac{t\bt}{4C_0}\right)^{\frac{q}{q-1}}(2\ga)^{-\frac{1}{q-1}}.
\end{equation}
Raising both sides of inequality (\ref{3.12}) to the power $\frac{1}{q-1}$ and taking into account the inequality $x^r\le x$ for $r\ge 1$, $0\le x\le 1$, we obtain
$$
a_m^{\frac{1}{q-1}} \le a_{m-1}^{\frac{1}{q-1}} (1-a_{m-1}^{\frac{1}{q-1}}/B).
$$
By Lemma 2.16 from \cite{Tbook}, p. 91, it implies
$$
a_m^{\frac{1}{q-1}} \le B/m,\qquad a_m\le B^{q-1}m^{1-q}.
$$
Representation (\ref{3.13}) gives
$$
B^{q-1} \le C(t,\ga,C_0)\bt^{-q}.
$$
This completes the proof of Theorem \ref{T3.1}.
\end{proof} 

In a certain sense, Theorem \ref{T3.1} complements Theorem \ref{T1.2} from Introduction. 
Theorem \ref{T1.2} gives the rate $m^{1-q}$ for any dictionary under assumption that a point $x^*$ of minimum: $E(x^*)=E^*$ belongs to $A_1(\D)$. Theorem \ref{T3.1} gives the same rate $m^{1-q}$ without any assumption on $x^*$ but with an assumption on $\D$ that $\beta(\D,X)>0$. 

{\bf 3.3. A bound for the atomic norm.} We show how Theorem \ref{T3.1} can be derived from Theorem \ref{T1.2} from Introduction. We note that Theorem \ref{T1.2} was proved in \cite{Tco1} for both the WCGA(co) and the WGAFR(co).
Application of Theorem \ref{T1.2} is based on the following characteristic
$$
R_1:=R_1(\D,X):= \sup_{\|x\|_X\le 1} \|x\|_{A_1(\D)},
$$
where $ \|x\|_{A_1(\D)}:=\inf\{c:x/c\in A_1(\D)\}$ is the atomic norm of $x$ with respect to $\D$.
We begin with a result relating characteristics $R_1(\D,X)$ and $\bt(\D,X)$ (see (\ref{2.4})) for a general not necessarily finite dimensional Banach space $X$. 
\begin{Theorem}\label{T3.3} Let $X$ be a uniformly smooth Banach space. Suppose a dictionary $\D$ is such that $\bt(\D,X)>0$. Then
$$
R_1(\D,X) = \bt(\D,X)^{-1}.
$$
\end{Theorem}
\begin{proof} We begin with the upper bound $R_1\le\beta^{-1}$. We give an algorithmic way of construction of a representation $x=\sum_i c_ig_i$ such that
$$
\sum_i|c_i| \le (1+\epsilon)\beta(\D,X)^{-1}\|x\|.
$$
It is clear that we have $R_1(\D^\pm,X)=R_1(\D,X)$ and $\bt(\D^\pm,X)=\bt(\D,X)$. Therefore,
without loss of generality we may assume that $\D$ is a symmetric dictionary: $g\in\D$ implies $-g\in\D$. 
 We use the following greedy algorithm (see, for instance, \cite{Tbook}, Section 6.7.3). 
Let $X$ be a uniformly smooth Banach space with modulus of smoothness $\rho(u)$, and let $\mu(u)$ be a continuous majorant of $\rho(u)$: $\rho(u)\le\mu(u)$, $u\in[0,\infty)$ such that $\mu(u)/u$ goes to $0$ monotonically.  For a symmetric dictionary $\D$ denote
$$
r_\D(f):=\sup_{g\in\D}F_f(g).
$$

 {\bf Dual Greedy Algorithm with parameters $(t,b,\mu)$ (DGA$(t,b,\mu)$).}
Let $X$ and $\mu(u)$ be as above.   For parameters $t\in(0,1]$, $b\in (0,1)$ and $f_0\in X$ we define sequences
$\{f_m\}_{m=0}^\infty$, $\{\ff_m\}_{m=1}^\infty$, $\{c_m\}_{m=1}^\infty$ inductively.   If for $m\ge 1$ $f_{m-1}=0$ then we set $f_j=0$ for $j\ge m$ and stop. If $f_{m-1}\neq 0$ then we conduct the following three steps.

(1) Take any $\ff_m \in \D$ such that
$$
F_{f_{m-1}}(\ff_m) \ge tr_\D(f_{m-1}).  
$$

(2) Choose $c_m>0$ from the equation
$$
\|f_{m-1}\|\mu(c_m/\|f_{m-1}\|) = \frac{tb}{2}c_mr_\D(f_{m-1}).  
$$

(3) Define
$$
f_m:=f_{m-1}-c_m\ff_m.  
$$

The following theorem (see, for instance, \cite{Tbook}, p. 365) guarantees convergence of the DGA$(t,b,\mu)$.

\begin{Theorem}\label{T3.4} Let $X$ be a uniformly smooth Banach space with the modulus of smoothness $\rho(u)$ and let $\mu(u)$ be a continuous majorant of $\rho(u)$ with the property $\mu(u)/u \downarrow 0$ as $u\to +0$. Then, for any $t\in (0,1]$ and $b\in (0,1)$ the DGA$(t,b,\mu)$ converges for each dictionary $\D$ and all $f_0\in X$.
\end{Theorem}

Take an $f_0\in X$ and apply the DGA$(t,b,\mu)$ to it. By Theorem \ref{T3.4} it gives a convergent expansion
$$
f_0 = \sum_{k=1}^\infty c_kg_k,\quad c_k>0,\quad g_k\in \D.
$$
It is easy to see from the choice of $c_k$ that  (see, for instance, \cite{Tbook}, p. 370--371)
$$
t(1-b)c_kr_\D(f_{k-1})\le \|f_{k-1}\|-\|f_k\|
$$
and
$$
\sum_{k=1}^\infty c_kr_\D(f_{k-1}) \le (t(1-b))^{-1}\|f_0\|.
$$
Next,
$$
r_\D(f_{k-1})\ge \bt(\D,X).
$$
Therefore, for any $t\in (0,1)$ and $b\in (0,1)$ we get
$$
\sum_{k=1}^\infty c_k \le (t(1-b))^{-1}\|f_0\|\bt(\D,X)^{-1}.
$$
This implies the upper bound in Theorem \ref{T3.3}. 

We now proceed to the lower bound. Suppose $R_1<\infty$. Fix $\e>0$. Then for any $F$, $\|F\|_{X^*}=1$, there exists $x_\e$ such that $F(x_\e)\ge 1-\e$, $\|x_\e\|=1$. Let 
$$
x_\e=\sum_i c^\e_ig^\e_i
$$
with $c^\e_i$ satisfying
$$
\sum_i |c^\e_i| \le R_1(1+\e).
$$
Further
$$
1-\e \le F(x_\e) = \sum_ic^\e_i F(g^\e_i) \le \sum_i|c^\e_i| \max_i |F(g^\e_i)| \le R_1(1+\e)\max_i |F(g^\e_i)|.
$$
Therefore
$$
\max_i |F(g^\e_i)| \ge R_1^{-1} \frac{1-\e}{1+\e}.
$$
This implies
$$
\beta(\D,X) \ge R_1^{-1}.
$$

\end{proof}

The assumption $\bt(\D,X)>0$ in Theorem \ref{T3.1} implies that for any $f_0$ such that $\|f_0\|\le C_0$ we get by Theorem \ref{T3.3}
$$
f_0\bt(\D,X)C_0^{-1} \in A_1(\D).
$$
Applying Theorem \ref{T1.2} with $\epsilon =0$, $B=C_0\bt(\D,X)^{-1}$ we obtain Theorem \ref{T3.1}.

\end{document}